\documentclass[runningheads]{llncs}
\usepackage{thm-restate}
\usepackage{fullpage}
\newcommand{\e}{\epsilon}
\usepackage{amsmath}
\usepackage{graphicx}
\begin{document}
\title{Fair Algorithm Design: Fair and Efficacious Machine Scheduling}
%
%
\author{April Niu\orcidID{0000-0002-6035-4850} \and
Agnes Totschnig\orcidID{0009-0004-7716-8183} \and
Adrian Vetta\orcidID{0000-0002-2213-4937}}
\authorrunning{A. Niu et al.}
%
\institute{McGill University, Canada\\
\email{yuexing.niu@mail.mcgill.ca, agnes.totschnig@mail.mcgill.ca, adrian.vetta@mcgill.ca}}
\maketitle              
\begin{abstract}Motivated by a plethora of practical examples where bias is induced by automated decision-making algorithms, there has been strong recent interest in the design of fair algorithms. However, there is often a dichotomy between {\em fairness} and {\em efficacy}: fair algorithms may proffer low social welfare solutions whereas welfare optimizing algorithms may be very unfair. This issue is exemplified in the machine scheduling problem where, for $n$ jobs, the social welfare of any fair solution may be a factor $\Omega(n)$ worse than the optimal welfare. In this paper, we prove that this dichotomy between fairness and efficacy can be overcome if we allow for a negligible amount of bias: there exist algorithms that are both ``almost perfectly fair" and have a constant factor {\em efficacy ratio}, that is, are guaranteed to output solutions that have social welfare within a constant factor of optimal welfare. Specifically, for any $\epsilon>0$, there exist mechanisms with efficacy ratio $\Theta(\frac{1}{\epsilon})$ and where no agent is more than an $\epsilon$~fraction worse off than they are in the fairest possible solution (given by an algorithm that does not use personal or type data). Moreover, these bicriteria guarantees are tight and apply to both the single machine case and the multiple machine case. The key to our results is the use of {\em Pareto scheduling mechanisms}. These mechanisms, by the judicious use of personal or type data, are able to exploit Pareto improvements that benefit every individual; such Pareto improvements would typically be forbidden by fair scheduling algorithms designed to satisfy standard statistical measures of group fairness. We anticipate this paradigm, the judicious use of personal data by a fair algorithm to greatly improve performance at the cost of negligible bias, has wider application.
\end{abstract}

\section{Introduction}\label{sec:intro}
The boom in automated decision-making is producing transformative effects on society. In principle, automation should produce lower cost and higher performance outcomes. Moreover, a common presumption was that the use
of algorithms based upon big data sets would eliminate the inherent bias in human decision-making. Unfortunately, that has not been the case. Indeed, it appears that in many instances automated decision-making has allowed bias to emerge on a huge scale and also to institutionalize partiality. 
A particularly nefarious example concerns judicial sentences based upon algorithms that predict the risk of reoffending~\cite{KK17,Eijk17}; these risk scores have been claimed to exhibit strong racial bias. 
Two more classical examples concern bias in the assignment of
credit scores~\cite{LDR18,Hassani21} and in automated hiring systems~\cite{RBK20}.
Indeed, as technology advances further, such as in facial recognition, the scope and potential dangers of automated decision-making are expanding dramatically~\cite{BG18,RB19}. 
Consequently, the issue of fair algorithms is now 
receiving a considerable amount of attention in the law, politics, philosophy, statistics and, especially, machine learning communities.

But what is a fair algorithm?
There is a subtlety here. There are two predominant, but quite distinct, ways to view fairness. First, an algorithm may be considered fair if it produces outcomes that are unbiased across groups. This is the underlying framework when fairness measures based 
on statistical criteria are applied. Second, an algorithm may be considered fair if its internal workings are unbiased. In particular, such an algorithm cannot use any personal or type information (such as race, sex, sexual orientation, wealth, class, etc.) nor use any correlated attributes.
Our approach is to combine these two viewpoints: we present a measure to assess fairness across groups (including individual fairness); furthermore, this measure is calibrated against the fairest outcome 
achievable by a mechanism that is not allowed to use personal information.

But our purpose is not to proffer a new measure of fairness; there is already an abundance of such definitions in the literature~\cite{BHN19}.
Nor is our aim simply to design an algorithm that computes fair solutions with respect to the measure. Our objective is more ambitious. We desire fair algorithms that also guarantee high performance. The grail is an algorithm that produces outcomes that are both unbiased {\bf and} efficacious (that is, of high social welfare).
Specifically, we are searching for an algorithm $\mathcal{A}$ that outputs
solutions that are (i) comparable in performance to those produced by the social welfare maximizing algorithm $\mathcal{A}^*$, and (ii)
comparable in fairness to those produced by the fairest algorithm $\mathcal{A}^f$. The reason that our proposed measure is useful is because it will naturally allow us to make such bicriteria comparisons. 

The reader may question whether this objective is achievable since there is often a dichotomy between fairness and social welfare. 
For proof of concept, in this paper we examine the machine scheduling problem.
In machine scheduling, personal type data concerns properties of the job
to be scheduled, specifically, size. (We remark that whilst size may appear a somewhat trifling characteristic in comparison to race or sex it does have important real-world consequences. For example, net-neutrality is based upon the idea that bandwidth allocation algorithms should not discriminate against agents based upon their size.)
Unfortunately, we will see that machine scheduling does indeed suffer from the aforementioned dichotomy: a fair scheduling algorithm that does not use the type data of an agent can produce arbitrarily poor outcomes from a social welfare perspective. Conversely, an optimal scheduling algorithm for social welfare can be arbitrarily unfair. This result appears fatal for our project.
But, amazingly, this dichotomy can be circumvented
by allowing for a {\em negligible} amount of bias. Specifically, we prove that very good social outcomes (within a constant factor of optimal)
can be obtained by algorithms that are near-perfectly fair (within a $1+\epsilon$ factor of perfect fairness, for any constant $\epsilon>0$).

\subsection{Overview and Results}
In Section~\ref{sec:background} we provide a brief history concerning the need for fair algorithms and discuss related works in the field.
In Section~\ref{sec:single} we present the classical machine scheduling problem with $n$ jobs (agents). We first explain how the optimal algorithm for social welfare $\mathcal{A}^*$ and the fairest algorithm $\mathcal{A}^f$ perform in the special case of a single machine. We then prove a strong dichotomy between fairness and efficacy by showing the fairest solution
can have social cost a factor $\Omega(n)$ worse than that of the optimal solution (see Example I). This motivates the study of $\epsilon$-fair mechanisms.
These mechanisms have the property that no individual is allowed to be more than an $\epsilon$ fraction worse off than they would be if the fairest algorithm was used. Since the guarantee applies at the individual level it also applies to any group as a whole.

In Section~\ref{sec:Pareto} we present a class of scheduling algorithms called {\em priority scheduling mechanisms}. We prove that it suffices to consider only a sub-class of these mechanisms, namely {\em Pareto scheduling mechanisms}, when searching for an algorithm that is both fair and efficacious. Specifically, these mechanisms can be used to output Pareto optimal solutions.\footnote{An outcome is Pareto optimal if no other outcome exists where at least one agent is better off and no agents are worse off.}
Indeed, the focus on Pareto solutions is also necessary: algorithms based upon statistical group fairness measures may output poor social solutions precisely because they are, typically, forced to output solutions that are Pareto dominated.

Then in Section~\ref{sec:ratio}, given a target fairness $\epsilon>0$, we prove that there is always a
Pareto optimal scheduling mechanism that is both $\epsilon$-fair and 
provides social cost within a factor of $\frac{1}{4\epsilon}$ of optimal.
Furthermore, we show that these bicriteria guarantees are the best possible.
Thus, by allowing a negligible amount of bias, high quality social outcomes can be guaranteed. This is somewhat analogous to differential privacy where, essentially, an individual's data can be included in a database if it does not comprise her individual privacy by more than an $\epsilon$ fraction, but with the bonus here that (in using personal data judiciously) we also obtain strong social welfare guarantees.
Finally, in Section~\ref{sec:multiple} we show how to extend these results to the setting of multiple machines with exactly the same bicriteria guarantees.

\section{Background and Related Work}\label{sec:background}

The potential for bias in automated decision-making systems has become a 
critical issue recently in a range of disparate applications.
Prominent examples include credit score evaluations~\cite{BN21,Hassani21,LDR18}, automated recruitment~\cite{RBK20,LT19}, judicial sentencing~\cite{KK17,Eijk17,Propublica161,Propublica162}, community policing~\cite{LI16,EFN18}, advertising~\cite{LT19}, medical treatments~\cite{OPV19,Igoe21}, social services~\cite{CBF18}, etc.
Moreover, these issues have arisen in even the largest and most influential
technology companies.
For example, Amazon's recruitment engine was found to prefer men over women~\cite{dastin18} and its same-day delivery algorithms produced outcomes that disadvantaged communities with large ethnic minority groups~\cite{Bloomberg16}.  
Gender stereotypes abound in Google News~\cite{BCZ16}. Topically, Twitter recently announced that its amplification algorithms are severely biased in favour of right-wing politicians~\cite{HKB22}.
Bias in facial recognition software could in the future have calamitous ramifications, so it is worrying that the commercial gender classifiers of IBM, Microsoft, and Megvii (Face++) perform the worst with darker skin females and better with males than females~\cite{BG18,RB19}.  
Another example is given by the advertisement algorithms of Facebook; these are specialized for target audiences, leading to lawsuits for violations of multiple civil rights laws~\cite{SGKMR19}.

Given the breadth and importance of these applications, the issue of algorithmic fairness has consequently received a considerable amount of attention in the law, politics, philosophy, statistics and, especially, machine learning communities.
In this section we will provide only a sample overview of the literature and also highlight the prior research most pertinent to our work; we refer the curious reader to the the book {\em Fairness and Machine Learning} \cite{BHN19} for a much more extensive and technical guide to the existing literature. 

The first problem facing the academic community is to decide whether or not an algorithm is fair. Therefore, a major focus has been to devise tests based upon measures of fairness.
These have been predominantly group-based statistical measures including the
influential measures of {\em independence}, {\em separation} and
{\em sufficiency}. The origins of these criteria date back to the works of
Darlington~\cite{Darlington71} and Cleary~\cite{Cleary68}.
These measures are commonly termed
{\em group fairness/statistical parity} ~\cite{DHP11}, {\em equalized odds} ~\cite{HPS16}, and {\em calibration within groups}
\cite{chouldechova16}, respectively.
Furthermore, there are now dozens of fairness criteria in the computer science literature closely related to just independence, separation and sufficiency 
alone; see Barocas et al.~\cite{BHN19}. Follow-up works based on group and individual fairness definitions abound. Dwork et al.~\cite{DIJ20} studied the composition of cohort pipelines under various individual fairness conditions. The notion of multi-group fairness is defined in Rothblum and Yona~\cite{RY21} where a multi-group agnostic PAC learning algorithm is proposed such that the loss on every sub-population is not much worse than the minimal loss for that population. A reduction from multi-group agnostic PAC learning to outcome indistinguishable learning \cite{DKRRY21} is used to obtain such an algorithm. In addition, in Hebert-Johnson et al.~\cite{HKRR18}, it is argued that calibration is not enough for achieving fairness; multicalibration is proposed to compute identifiable subpopulations within a collection, on which learning algorithms output predictions that are highly accurate. This model is further used in predicting COVID-19 mortality risk where multicalibration is combined with other baseline models~\cite{BRALFYGSSBRSNBD20}. Other measures based on 
causal reasoning and the use of counterfactuals have also been proposed~\cite{KLR17,KRPHJS17,NS18}. A stronger theoretical result~\cite{DIRS20} gives an extraction procedure capable of learning from a fairness oracle with an arbitrary fairness condition.


A consequence of the abundance of tests is that these measures may be mutually incompatible~\cite{Darlington71}; see also Kleinberg et al.~\cite{KMR17} and Chouldechova~\cite{chouldechova16}. This can make assessments of bias somewhat subjective. Consider, for example, {\em risk assessment scores}. These are widely used in the criminal justice system in the US in assigning bail or remand (pre-trial detention)~\cite{KK17}. ProPublica~\cite{Propublica161,Propublica162} found that {\sc compas} risk scores, generated using software by Northpointe, were strongly biased against African-Americans. Specifically, African-Americans were twice as likely to be falsely flagged as high risk, while Caucasian defendants were more often mislabeled as low risk. 
Northpointe~\cite{dieterich2016compas} counterargued these findings by claiming that the {\sc compas} risk scores were fair because they satisfy the sufficiency criteria.
More detailed discussion on this case and on the incompatibilities and trade-offs between popular statistical fairness criteria can be found in Berk et al~\cite{BHS17}. 

The shift from human to automated decision-making has also added a degree of opacity, as most predictors lack interpretability, tracability and auditability. We tend to view technology as an impartial prediction tool that is less prone to making mistakes than the human judgments it replaces and, consequently, assign it unwarranted legitimacy. Another major issue with black-box predictors is the diffusion of responsibility and the loss of accountability that ensues from a lack of clarity in the automated decision process~\cite{LI16}. The trade-offs between different predictors with respect to interpretability, accuracy and fairness are discussed in Chouldechova et al.~\cite{CBF18} for the case of child maltreatment hotline screening processes. A more general discussion about screening decisions and fairness can be found in Rambachan
et al.~\cite{RKLM20}.

In addition to modelling and assessing fairness, a major research agenda
concerns the practical issues that arise with real-world data.
These include the collection, appropriate usage, and dynamic maintenance of data.
Consider the findings of Holstein et al.~\cite{HWD19} based upon interviews with ML industrial practitioners. In practice, datasets are often both incomplete and inappropriate for assessing fairness. Thus guidance is needed in the data collection
stage. Moreover, even if important sensitive data (such as type data) has been collected, it may be withheld from the industrial teams actually developing and running the algorithms, making accurate auditing impossible. Furthermore, humans are involved throughout the development pipeline (including in data collection and data labelling for training)
leading to the possible incorporation of human bias. Also, unlike in hiring or judicial sentencing, in many industrial applications the outcomes and objectives are more fuzzy and opaque (e.g. chatbots, tutoring, etc) so are less conducive for evaluation by standard statistical fairness measures.
For such practical reasons, thorough auditing and documentation is proposed as  essential in remedying bias in computer programs~\cite{RBK20,MWZ19}.

As alluded to, great care has to be taken with real-world data because there may be many components that are highly correlated with a critical or protected attribute.
In particular, the naive approach of ``{\em fairness through unawareness}'', where the protected attribute is removed, is insufficient to guarantee fair treatment~\cite{RD20}. 
The dynamic nature of data collection can itself be problematic, for example, in causing negative feedback loops, which can amplify the bias contained in the data sets. Consider a case study on predictive policing in Oakland conducted by Lum and Isaac \cite{LI16}: a software recommendation of increased policing in a region will lead to an increase in recorded crimes in the dataset; this, in turn, can lead to a recommendation of further increases in policing for that area, etc; see Ensign et al.~\cite{EFN18} for a mathematical model of this phenomenon.
Negative feedback loops may also arise in automated hiring systems, where positive information may be added to the database for successful candidates whereas no or negative information may be added for candidates the software rejected~\cite{wachter17}.
Similar feedback-loop effects have also been observed in credit score calculations~\cite{BN21} and in the polarization of social networks~\cite{SLL21}.


Let us conclude this overview by discussing the fairness literature most pertinent to our work. Mullainathan~\cite{M18} discusses the relationships between social welfare and fairness. Dwork et al.~\cite{DHP11} highlight two aspects of particular relevance here. One, they emphasize the importance of
fairness measures that extend beyond group fairness and apply at the level of an individual. They do this using a Lipschitz condition on the classifier.
Two, they propose a bicriteria approach, namely utility optimization subject to
a fairness constraint. 
Our approach follows this framework. In essence, we wish to optimize welfare (minimize social cost) subject to individual fairness constraints, namely $\epsilon$-fairness.

Moreover, we further desire quantitative bicriteria performance guarantees. This is
analogous to the concept of the {\em price of fairness} introduced independently by Caragiannis et al.~\cite{CKK12} and
Bertsimas et al.~\cite{BFT11}. 
The price of fairness quantifies the loss in social welfare (economic efficiency) incurred under the fairness scheme. Specifically, for a minimization problem, it is the worst case ratio between the cost of the ``best'' fair solution and the cost of the most efficient solution.\footnote{Bertsimas et al.~\cite{BFT11} actually use a measure set equal to one minus this ratio, but that is of no significance here.}
We remark that this approach has wide application depending upon the definition of fairness used. For example, in Caragiannis et al.~\cite{CKK12} the problem studied is the classical economic problem concerning the fair allocation of a collection of goods, where an outcome
is deemed fair if it satisfies the properties of proportionality, envy-freeness, or equitability. For proof of concept, we have chosen to study the machine scheduling problem with $n$ agents and $m$ machines. The price of fairness in machine scheduling has been studied for the special case of two-agents on a single machine~\cite{ACN19,ZZ20}.

Our main result can be viewed through the price of fairness lens. Specifically, for machine scheduling, the price of fairness for the class of perfectly fair ($0$-fair) schedules is $\Omega(n)$ but the price of fairness for the class of near-perfectly fair ($\epsilon$-fair) schedules is only $\Theta({1}/{\epsilon})$ for any $\epsilon>0$.

\section{The Machine Scheduling Problem}\label{sec:single}

In the {\em machine scheduling problem} there are $n$~agents who wish to schedule
a single job each on one of $m$~machines. 
Agent~$i$ has a job of {\em size} (duration) $d_i\ge 0$ and its objective is to minimize
the {\em completion time} $c_i(\mathcal{A})$ of its job, which is a function of the
(possibly, randomized) assignment mechanism $\mathcal{A}$ used to assign the jobs to the machines. The {\em social cost} of an assignment is the sum of the (expected) completion times.
Our objective is to design a {\em fair} allocation mechanism that performs well in comparison to the optimal mechanism for social cost. To begin, in Sections~\ref{sec:single} to~\ref{sec:ratio}, we will focus on the case of a single machine, namely $m=1$.
Using the lessons derived from the single machine case, we will study the general case of 
multiple machines, that is $m\ge 2$, in Section~\ref{sec:multiple}. 

\subsection{Fair Schedules and Optimal Schedules}\label{sec:definitions}

Consider the case of a single machine. Here the completion time for agent $i$ is simply the sum of its size plus the sizes of the jobs that are scheduled on the machine before it is.
It is now easy to compute both the fairest schedule and the optimal schedule.
The optimal mechanism $\mathcal{A}^*$ for social cost is given by {\tt Smith's Rule}~\cite{Smi56}: {\em schedule the jobs in increasing order of size}. On the other hand, recall that a fair mechanism may not use any private characteristic of the agent.
In this case, the only such characteristic is size. Thus, for machine scheduling a fair mechanism must treat the agents identically. 
In particular, for any pair of jobs, job $i$ should be equally likely to appear before or after job $\ell$.
Ergo, the fair solution is given by the
{\tt Random Assignment Rule}: {\em schedule the jobs in random order}.
We remark that for machine scheduling on a single machine there is a unique fair solution.
Hence for this application it makes sense to refer to the random assignment rule as the {\em fairest mechanism} $\mathcal{A}^f$. 
(For the multiple machine setting we will see that there is more than one fair mechanism. So, for the purpose of comparison, there we define the fairest mechanism $\mathcal{A}^f$ to be the mechanism with the best social welfare among the set of all fair mechanisms.)

We may now calculate the social cost of the optimal assignment mechanism
$\mathcal{A}^*$. Without loss of generality, label the agents such that
$d_1\le d_2\le \cdots \le d_n$. Then, under Smith's Rule, agent $i$ has
completion time $c_i(\mathcal{A}^*) = \sum_{\ell=1}^i d_{\ell}$. 
It follows immediately that the social cost of the optimal assignment is
\begin{equation}\label{eq:opt}
  c(\mathcal{A}^*) 
  \ =\ \sum_{i=1}^n c_i(\mathcal{A}^*) 
  \ =\ \sum_{i=1}^n \sum_{\ell=1}^i d_{\ell}
  \ =\ \sum_{i=1}^n (n-i+1)\cdot d_i.
\end{equation}
Next, let's compute the expected social cost of the fairest mechanism $\mathcal{A}^f$.
Define $D=\sum_{i=1}^n d_i$ to be the sum of the job sizes.
Under the randomized assignment mechanism, job $\ell$ will appear before
job $i$ with probability exactly~$\frac12$. Thus the expected completion
time of job $i$ is exactly
\begin{equation}\label{eq:fair-i}
  c_i(\mathcal{A}^f) 
  \ =\ d_i+\sum_{\ell\neq i} \frac12\cdot d_{\ell}
  \ =\ \frac12\cdot (D+ d_i ).
\end{equation}
Consequently, the expected social cost of the fairest mechanism is
\begin{equation}\label{eq:fair}
  c(\mathcal{A}^f) 
  \ =\ \sum_{i=1}^n c_i(\mathcal{A}^f)
  \ =\ \sum_{i=1}^n \frac12\cdot (D+ d_i )
  \ =\ \frac12\cdot \sum_{i=1}^n D + \frac12\cdot\sum_{i=1}^n d_i
  \ =\ \frac12 D\cdot (n+1)  .
\end{equation}
We can now formalize how to measure the trade-off between fairness and efficacy. A standard approach to do this is via the {\em price of fairness}, due to
Caragiannis et al.~\cite{CKK12}; see also Bertsimas et al.~\cite{BFT11}.
This is the maximum ratio between the social cost of the best fair solution ($\min_{S\in \mathcal{F}(I)} c(S)$) and that of the optimal solution $c(S^*)$ over all possible instances $\mathcal{I}$. Formally, 
let $\mathcal{I}$ be the set of instances and let $\mathcal{F}(I)$ be
the set of solutions that satisfy the proscribed fairness criteria for an instance $I\in \mathcal{I}$.
Then the price of fairness is
\begin{equation}\label{eq:PoF-1}
  \Phi \ = \ \max_{I\in \mathcal{I}} \, \min_{S\in \mathcal{F}(I)} \, \frac{c(S)}{c(S^*)}.
\end{equation}
We may also define the {\em price of fairness of a mechanism} $\mathcal{A}$
as the worst case ratio, over all possible instances, of the social cost of the solution output by the mechanism compared to that of the optimal mechanism $\mathcal{A}^*$.
\begin{equation}\label{eq:PoF-2}
  \Phi(\mathcal{A}) \ = \ \max_{I\in \mathcal{I}} \ \frac{c(\mathcal{A})}{c(\mathcal{A}^*)}.
\end{equation}
We also call $\Phi(\mathcal{A})$ the {\em efficacy ratio} (or {\em social welfare ratio}) of the mechanism $\mathcal{A}$. For machine scheduling, we will see in Section~\ref{sec:Pareto} that we can compute (\ref{eq:PoF-1}) from (\ref{eq:PoF-2}). That is, under the fairness measure considered, we can find a mechanism $\mathcal{A}$ that for any instance $I$ outputs the best fair solution in $\mathcal{F}(I)$.

\subsection{The Dichotomy between Fairness and Efficacy}\label{sec:bad-example}
Unfortunately, the following example shows that achieving an assignment mechanism that is {\bf both} fair and efficacious is a chimera.

\begin{quote}
{\tt Example I.} Let there be $n-1$ small jobs of size $1$ and one large job of size
$d_n\gg 1$. By~(\ref{eq:opt}), the optimal social cost is 
\begin{equation}\label{eq:bad-opt}
c_i(\mathcal{A}^*) 
\ =\ \sum_{i=1}^n (n-i+1)\cdot d_i
\ =\  d_n+ \sum_{i=1}^{n-1} n-i+1
\ =\ d_n+ \sum_{j=2}^{n} j
\ =\ d_n+ \frac12n(n+1)-1.
\end{equation}
In contrast, by (\ref{eq:fair}), the social cost of the fairest mechanism is
\begin{equation}\label{eq:bad-fair}
c_i(\mathcal{A}^f) 
\ =\ \frac12 D\cdot (n+1) 
\ =\  \frac12\cdot (n-1+d_n)\cdot (n+1)
\ =\ \frac12(n+1)\cdot d_n +\frac12(n^2-1).
\end{equation}
But, for large $d_n$, we see from (\ref{eq:bad-opt}) and (\ref{eq:bad-fair}) that 
$\frac{c(\mathcal{A}^f)}{c(\mathcal{A}^*)}$ tends to $\frac12(n+1)$.
Thus the price of fairness for machine scheduling is at least $\frac12(n+1)$.
\end{quote}
This example is troubling as a price of fairness of $\Omega(n)$ is the worst possible.
Seemingly, one may conclude that this rules out any possibility of designing a mechanism for scheduling that is both fair and effective.
The remainder of the paper is devoted to proving that this conclusion 
is emphatically incorrect!

\subsection{Near-Fair Mechanisms}

To show the existence of a fair and efficacious scheduling mechanism, our task is two-fold. First, we desire a scheduling algorithm $\mathcal{A}$ that outputs a solution whose social cost is comparable to that given by the optimal algorithm
$\mathcal{A}^*$.
Second, the fairness of the solution output by $\mathcal{A}$ must also be comparable to that given by the fairest algorithm $\mathcal{A}^f$. 

For the former desiderata it suffices to prove that the efficacy ratio $\Phi(\mathcal{A})$ is small, in particular, our target is a constant
efficacy ratio. For the latter desiderata we require a measure of fairness.
Our measure will be at the level of the individual.
Specifically, for a given instance $I\in \mathcal{I}$, we say that the fairness $f_i(\mathcal{A})$ of the mechanism to agent $i$ is the ratio between the expected completion time $c_i(\mathcal{A})$ and the completion time $c_i(\mathcal{A}^f)$ it receives in the fairest solution.
We then define the fairness ratio $f(\mathcal{A})$ for the mechanism $\mathcal{A}$ to be the 
worst case fairness over all agents and over all instances.
\begin{equation*}\label{eq:f-i}
  f(\mathcal{A}) 
  \ = \ \max_{I\in \mathcal{I}} \, \max_i \ f_i(\mathcal{A})
  \ = \ \max_{I\in \mathcal{I}} \, \max_i \ \frac{c_i(\mathcal{A})}{c_i(\mathcal{A}^f)} 
\end{equation*}
For $\epsilon\ge 0$, we say the mechanism $\mathcal{A}$ is 
$\epsilon$-{\em near-fair}, or more concisely $\epsilon$-{\em fair}, if
$f(\mathcal{A})\le  1+\epsilon$.
This implies that an $\epsilon$-fair algorithm can never output a solution in which the cost to any agent~$i$ is more than an $\epsilon$~fraction greater than the cost it pays in the fairest solution. Thus if $\epsilon$ is large the use of type-data by the algorithm $\mathcal{A}$ significantly harms at least one individual; if $\epsilon$ is small the use of type-data does not significantly harm any individual.

In this paper, our target is for $\epsilon$-{\em fair} mechanisms
where $\epsilon$ is close to zero. 
We call such an algorithm {\em almost perfectly fair}. Observe that
a $0$-fair algorithm $\mathcal{A}$ is {\em perfectly fair} -- every agent does at least as well under $\mathcal{A}$ as under the fairest mechanism $\mathcal{A}^f$.
Unfortunately, Example I extends to all perfectly fair algorithms. Ergo, all $0$-fair algorithms have extremely poor efficacy ratios, namely $\Omega(n)$.
Remarkably, we can circumvent this negative result by proving that if we allow for an arbitrarily small amount of unfairness then the efficacy ratio falls to a constant!
Specifically, we will show that an $\epsilon$-fair algorithm exists for which the the worst case efficacy ratio is $\Theta(\frac{1}{\epsilon})$.
Our result then implies the existence of an algorithm that is able to use type-data in such as way as to massively improve the social performance of the algorithm whilst harming no individual more than a negligible amount.

We remark that, even without the social welfare guarantee, $\epsilon$-fairness itself is an extremely strong guarantee.
Because it applies at the individual level it automatically applies at the group level.\footnote{This important property is an immediate consequence of the definition of our measure. We remark that this property does not hold for statistical measures of fairness: typically individual fairness does not imply group fairness, and group fairness does not imply individual fairness. See Binns~\cite{Bin20} for detailed discussions on this issue.} 
Moreover, it simultaneously applies to every group no matter how they are defined; in particular this includes groups that a particular study may not even be investigating! We emphasize that the concept of near-fair mechanisms extends beyond applications where there is a unique fairest algorithm $\mathcal{A}^f$. 
Specifically, it suffices to demand that, for every agent, the mechanism provides a comparable utility to that given in expectation by a collection of fair algorithms $\{\mathcal{A}^{f_1}, \mathcal{A}^{f_2}, \dots, \mathcal{A}^{f_k}\}$.
Thus the bicriteria measure can essentially be applied to any application setting;
only a target set of criteria (or target set of algorithms/outcomes) are required for the purpose of comparison.

\section{Pareto Optimal Schedules}\label{sec:Pareto}
In this section, we show that we can refine our search space when looking 
for a scheduling mechanism that is both fair and efficacious. 
To wit, we define the class of priority scheduling mechanisms and prove it has a subclass, namely Pareto scheduling mechanisms, that correspond exactly to algorithms
that output Pareto optimal solutions. Ergo, it will suffice in Section~\ref{sec:ratio} to study only Pareto scheduling mechanisms to obtain the best possible performance guarantees
with respect to fairness and efficacy.

\subsection{Priority Scheduling Mechanisms}

The most general scheduling mechanism assigns a probability $p_{\pi}$ to 
each permutation $\pi$ of the agents and then
schedules the agents in the order $\pi$ with probability $p_{\pi}$. 
For our purposes it will suffice to consider only the class
of {\em priority scheduling mechanisms}. A priority scheduling mechanism
partitions the jobs into priority groups. Every job of priority (group) $k$ must be scheduled before every job of priority $k+1$.
Jobs of the same (priority) group are scheduled in random order.
For example, both Smith's rule and the randomized assignment rule are 
priority scheduling mechanisms. For the former, there are $n$ priority groups
with job $i$ alone in priority group $i$. For the latter, there is just one priority group, that is, every job has the same priority.

Our task now is to show that the priority scheduling mechanisms include the optimal mechanisms in terms of the trade-off between fairness and efficacy.

\subsection{Pareto Optimal Mechanisms} 

We say that scheduling mechanism $\mathcal{A}_1$ Pareto dominates
$\mathcal{A}_2$ if 
$f(\mathcal{A}_1)\le f(\mathcal{A}_2)$ and $c(\mathcal{A}_1)\le c(\mathcal{A}_2)$ (with at least one inequality strict).
That is, $\mathcal{A}_1$ is both fairer and more efficacious than $\mathcal{A}_2$. 
To begin, we wish to characterize the mechanisms that are not Pareto dominated.
These are the {\em Pareto optimal mechanisms} -- no mechanism exists which yields better fairness and better efficacy.
Thus these mechanisms lie on the {\em Pareto frontier} of all scheduling mechanisms (we will give an illustration of this later in Figure~\ref{fig:frontier}). Our aim is to show this Pareto frontier consists {\bf only} of the Pareto optimal mechanisms. To do this, we proceed by proving a series of properties held by Pareto optimal mechanisms.

\begin{lemma}\label{lem:fairness-last}
For any priority scheduling mechanism, the worst fairness applies to a job in the lowest priority group.
\end{lemma}
\begin{proof}
Take any pair of jobs $i$ and $j$ where $j$ has lower priority than $i$.
We wish to show that $j$ has worse fairness than $i$. It will become apparent from the subsequent proof that we may
assume that $j$ is in the last priority group and $i$ is in the penultimate priority group.
Let $A$ be the total duration of all jobs before the penultimate group, i.e. the starting time of the penultimate group. Let $B$ denote the total duration of all jobs in the penultimate group, and $L$ the duration of all jobs in the last group. 
Hence, $A + B + L = D$. Take any job $i$ in the penultimate group.
The fairness for $i$ is:
\begin{equation}\label{eq:fairness-of-i}
f_i 
\ =\ \frac{A + \frac{1}{2}(B + d_i)}{\frac{1}{2}(D + d_i)}
\ =\ \frac{2A + B +d_i}{D + d_i}
\ =\ \frac{D + A - L + d_i}{D + d_i}
\ =\ 1+ \frac{A - L}{D + d_i}.
\end{equation}
Next take any job $j$ in the last group.
The fairness for $j$ is:
\begin{equation}\label{eq:fairness-of-j}
f_j 
\ =\ \frac{A + B + \frac{1}{2}(L + d_j)}{\frac{1}{2}(D + d_j)}
\ =\ \frac{2A + 2B + L + d_j}{D + d_j}
\ =\ \frac{D+A+B + d_j}{D+ d_j}
\ =\ 1 + \frac{A + B}{D +d_j}.
\end{equation}
From (\ref{eq:fairness-of-i}) and (\ref{eq:fairness-of-j}), to show that $f_j > f_i$
it suffices to verify that $(A + B)(D + d_i)> (A-L)(D+ d_j)$.
Rearranging, we require that $Ad_i + B(D + d_i) + L(D + d_j)>Ad_j$.
To see this, observe that $d_j \le L$, since $d_j$ is the duration of a single task in the last group, and that $A < D$. Hence,
$Ad_j 
\ \le\ AL
\ <\ DL
\ <\ Ad_i + B(D + d_i) + L(D + d_j).$
It follows that fairness strictly increases between groups and, in particular, that jobs in the last priority group have the worst fairness.
\qed
\end{proof}

Lemma~\ref{lem:fairness-last} has an important consequence. Because,
the overall fairness of the schedule is determined only by the lowest priority grouping,
we obtain a Pareto improvement if every job of higher priority is ordered
by size in its own singleton priority group. As this coincides with 
Smith's rule for those items, this can only decrease the social cost.
Further, by Lemma~\ref{lem:fairness-last}, because it has no effect on the lowest
priority grouping, the fairness of the schedule remains the same. Henceforth, we need only consider schedules where the
lowest priority group has a random ordering and, before them, all other items
are ordered by Smith's rule. For ease of exposition, we will also
partition the jobs into two {\em sections}. The second section consists
of jobs in the lowest priority group; the first section consists of
all the other jobs. Thus jobs in the first section are ordered by
Smith's rule and after them the jobs in the second section are ordered randomly.

Now we have seen that the jobs in the lowest priority group have the worst fairness. Amongst them, the smallest job has the worst fairness.
\begin{lemma}\label{lem:last-shortest}
For any priority scheduling mechanism, the worst fairness applies to the smallest job in the lowest priority group.
\end{lemma}
\begin{proof}
Let $A$ denote the starting time of the last group. Take any job $i$ in this lowest priority group. The fairness for $i$ is:
\begin{equation}\label{eq:worst-fairness}
f_i 
\ =\ \frac{A + \frac{1}{2} (D-A + d_i)}{\frac{1}{2}(D + d_i)}
\ =\ \frac{2A + D - A + d_i}{D + d_i}
\ =\ \frac{D + A + d_i}{D + d_i}
\ =\ 1+ \frac{A}{D + d_i}.
\end{equation}
But this is largest when $d_i$ is smallest. Consequently, the worst case fairness is achieved by the smallest job in the lowest priority group. 
\qed
\end{proof}

Lemma~\ref{lem:last-shortest} also has an important consequence. It implies that
the lowest priority group must contain the largest jobs. More specifically,
a smaller job should never have higher priority than a larger job.

\begin{lemma}\label{lem:Pareto-optimal}
In any Pareto optimal schedule, a smaller job can never have higher priority than a larger job.
\end{lemma}
\begin{proof}
We have already shown jobs in the first section must be
ordered by Smith's rule. Thus, for a contradiction, let $d_i>d_j$ where
job $j$ is in the second section and job $i$ is in the first section. We may assume that $j$ is the smallest job in the second section and that $i$ is the largest job in the first section.
We will show that swapping jobs $i$ and $j$ leads to a Pareto improvement.

First consider the change in social cost caused by this swap.
Observe that the completion times of the jobs in section one (except $i$) are unaffected by the swap. Indeed these jobs just delay
the other jobs by a fixed amount (the sum of their sizes). Thus, for the purpose of comparison,
we may assume that their sizes sum to zero, indeed that there are no such jobs.
It then follows that social cost before the swap is:
\begin{equation*}
d_i + d_i \cdot (n-1) + \frac{D - d_i}{2}\cdot (n -1 +1) 
\ =\ d_i \cdot n + \frac{D - d_i}{2} \cdot n 
\ =\ \frac{D + d_i}{2} \cdot n.
\end{equation*}
The social cost after the swap is:
$d_j + d_j \cdot (n-1) + \frac{D - d_j}{2} \cdot (n -1 +1) 
\ =\ \frac{D + d_j}{2} \cdot n.$
By assumption, $d_i > d_j$, and so the social cost decreases after swapping the jobs.


Next consider the change in fairness after the swap.
Before the swap, by Lemma~\ref{lem:last-shortest} and (\ref{eq:worst-fairness}),
the fairness is:
$1+ \frac{A}{D + d_j}$.
After the swap, the fairness of any job $k$ of lowest priority is:
$1+ \frac{A-d_i+d_j}{D + d_k} \ <\  1+ \frac{A}{D + d_k} \ \le \ 1+ \frac{A}{D + d_j}.$
Here the strict inequality follow as $d_i>d_j$. The inequality holds as every
other job $k$ in section two after the swap (possibly $k=i$) is at least as large as $j$.
Thus the fairness also improves after the swap and we have a Pareto improvement,
a contradiction.
\qed
\end{proof}


Our arguments imply that there are only $n$ Pareto optimal scheduling mechanisms.

\begin{corollary}\label{cor:Pareto-schedules}
There are exactly $n$ priority scheduling mechanisms that are Pareto optimal.
\end{corollary}
\begin{proof}
By Lemma~\ref{lem:Pareto-optimal}, any Pareto optimal schedule must be consistent
with the order of the job sizes. Further, by Lemma~\ref{lem:fairness-last},
every priority group is a singleton except for the lowest priority group.
Thus to define a Pareto optimal schedule we may simply select any 
number $0\le k< n$. Then let each job $i\le k$ form a singleton group with priority $i$,
and let jobs $\{k+1, k+2,\dots n\}$ form a group with (lowest) priority~$k+1$.
We call this the $k$-th Pareto optimal schedule which we denote by $\mathcal{A}^k$.
It follows that there are $n$ Pareto optimal schedules: $\mathcal{A}^0, \mathcal{A}^1, \dots,\mathcal{A}^{n-1}$.
\qed
\end{proof}


\subsection{The Pareto Frontier}\label{app:a}
Let's see an example illustrating how the $n$ Pareto scheduling mechanisms lie on the Pareto frontier.
Figure~\ref{fig:frontier} shows the trade-off between fairness and social cost for an example with $2^{9-\ell}$ jobs of size $2^{\ell}$, for all $0\le \ell \le 9$. That is, there are ten job sizes $1, 2, 4,\dots, 512$ with $512, \dots, 4, 2, 1$ jobs of each size, respectively.
In particular, the total duration of all the jobs of each size is constant (512).
The $y$-axis gives the fairness guarantee, $\epsilon$, and the $x$-axis gives the efficacy ratio for the instance. Each point represents the performance of a given priority scheduling mechanism. 
But as there are an exponential number of priority scheduling mechanisms and we have $n=1023$ agents,
we have not plotted them all in Figure~\ref{fig:frontier}. Specifically, shown are only those
priority scheduling mechanisms that did not place two jobs of identical size in different
sections (that is, in both the first section and the second section). Of the 1023 Pareto optimal
mechanisms, ten satisfy this property and they are shown in red; observe that these do lie on the Pareto frontier.

\begin{figure}[h!]
\centering
\includegraphics[scale=0.8]{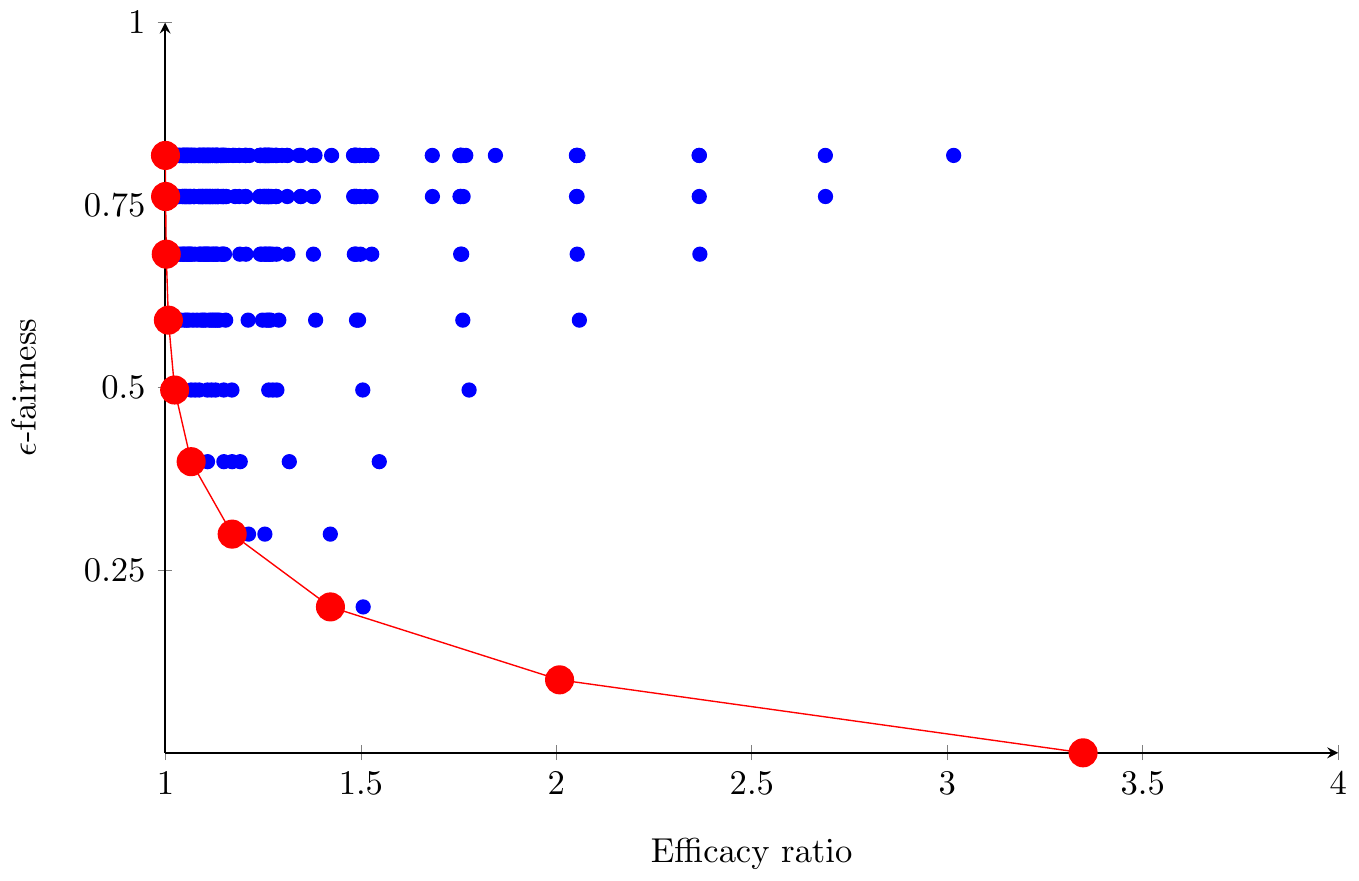}
\caption{The Pareto frontier showing the fairness vs social cost trade-off.}\label{fig:frontier}
\end{figure}

We remark that Lemma~\ref{lem:fairness-last} explains why many of the
mechanisms share the same fairness guarantee and, hence, are plotted on
the same horizontal line. In addition, in Figure~\ref{fig:frontier}, Smith's rule corresponds 
to the leftmost red Pareto schedule $\mathcal{A}^{n-1}$ and the random assignment rule corresponds 
to the rightmost red Pareto schedule $\mathcal{A}^{0}$.

\section{Fair and Efficacious Scheduling}\label{sec:ratio}

We must now decide on the appropriate trade-off between fairness and social cost. 
By the results in Section~\ref{sec:Pareto}, it suffices to analyze only the
$n$ Pareto optimal mechanisms $\{\mathcal{A}^0, \mathcal{A}^1, \dots,\mathcal{A}^{n-1}\}$.
If we run $\mathcal{A}^{n-1}$ then this is Smith's rule 
and thus, trivially, we can obtain the optimal solution. But this gives a fairness
guarantee where $\epsilon$ can be as large as $1$, which is extremely unfair. 
Consequently, our interest lies in smaller values of $\epsilon$. If we 
run $\mathcal{A}^{0}$ then this is the random assignment rule.
So we have $\epsilon=0$ giving a perfectly fair assignment.
But now we have the opposite problem, the social cost may be extremely high.
Specifically, from Example I, the efficacy ratio of this mechanism is $\Omega(n)$.
Can we get the best of both worlds? In particular, for a fixed $\epsilon>0$
what is the best social cost we can guarantee?

\subsection{The Fairness of Pareto Scheduling}

Let's commence by computing the fairness of each Pareto optimal scheduling mechanism, $\mathcal{A}^{k}$.
Take an instance $I$ with sizes $d_1\le d_2\le \cdots \le d_n$ and define 
\begin{equation*}
\epsilon_k \ =\ \frac{\sum_{\ell=1}^k d_\ell}{\sum_{\ell=1}^n d_\ell} \ =\ \frac{\sum_{\ell=1}^k d_\ell}{D}.
\end{equation*}

\begin{theorem} The $k$-th Pareto optimal scheduling mechanism $\mathcal{A}^k$ is $\epsilon$-fair, where $\epsilon=\epsilon_k$.
\label{thm:epfair}
\end{theorem}
\begin{proof}
Let job $i=k+1$ be the smallest job in the second section and set $\epsilon=\epsilon_k$. Then by Lemma~\ref{lem:last-shortest} the fairness of the
$k$-th Pareto optimal schedule to agent $i$ is
\begin{equation*}
f_i(\mathcal{A}^k) \ = \ \frac{\epsilon D + \frac{1}{2} \left[ (1- \epsilon) D + d_i \right]}{\frac{1}{2}(D + d_i)} \\
\ =\  \frac{2\epsilon D + (1 - \epsilon) D + d_i}{D + d_i} \\
\ =\  \frac{(1 + \epsilon) D + d_i}{D + d_i} \\
\ <\  \frac{(1 + \epsilon) D}{D} \\
\ =\  1 + \epsilon.
\end{equation*}
Thus the schedule is indeed $\epsilon$-fair.
\qed
\end{proof}


%

\subsection{The Efficacy of Pareto Scheduling}

So the $k$-th Pareto optimal schedule $\mathcal{A}^k$ is $\epsilon$-fair, for $\epsilon =\epsilon_k=\frac{\sum_{\ell=1}^k d_\ell}{\sum_{\ell=1}^n d_\ell}$. But what is its efficacy ratio? 
Take a worst-case instance $I$ with sizes $d_1\le d_2\le \cdots \le d_n$ where $\sum_{\ell=1}^k d_\ell= \epsilon\cdot D$.
To find its efficacy, we will transform $I$ into an auxiliary instance $\hat{I}$ whose efficacy is at least as bad as $I$, but which is easier to analyze.
To begin, by scaling, we may assume $d_k=1$.
We say that a job $i$ is {\em small} if $d_i<1$, is {\em medium} if $d_i=1$, and is
{\em large} if $d_i>1$.
Then, given $I$, define the following parameters: 
\begin{itemize}
    \item Let $S$ be the total size of the small jobs: i.e. $S=\sum_{j:d_j<1} d_j$.
    \item Let $M$ be the total size of the medium jobs, i.e. $M=\sum_{j: d_j=1} d_j = |\{j: d_j=d_k=1\}|$.
    \item Let $L$ be the total size of the large jobs, i.e. $L=\sum_{j:d_j>1} d_j$.
    \item Let $\lambda \in [0,1]$ be the fraction of the medium jobs in the first section.
\end{itemize}
These definitions are illustrated in Figure~\ref{fig:efficacy}. Observe that,
for the $k$-th Pareto optimal schedule,
the jobs corresponding to $S$ are all in the first section and the jobs corresponding to $L$ are all in the second section. 
Of course, Smith's rule will simply order the jobs by size.
\begin{figure}[h]
    \centering
    \includegraphics[scale=1]{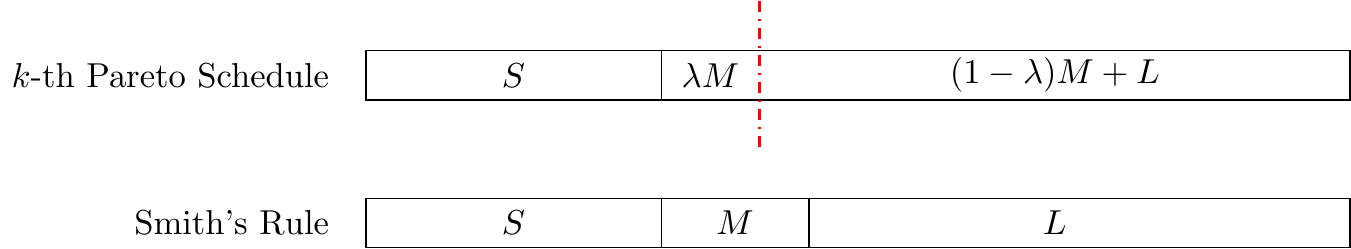}
\caption{Two different algorithms on instance $I$.}
\label{fig:efficacy}
\end{figure}
We are now ready to transform $I$ into an auxiliary instance $\hat{I}$. Furthermore, $\hat{I}$ will satisfy two properties that will simplify
our analyses.

{\sc Property 1.} Every small job in $\hat{I}$ has size $0$.

\begin{proof}
Take any small job $i$ in $I$. Then in $\hat{I}$ set $d_i=0$.
We need to prove this change cannot decrease the efficacy ratio.
To do this we require the following definitions.
For the original instance $I$, let $n_S, n_M=M$ and $n_L$ be the number of small, medium and large jobs, respectively. In addition, let $c^*(S)$
be the social cost if we apply Smith's rule only on the small items.
Similarly, let $c^*(M)$ and $c^*(L)$ be the social cost incurred under Smith's rule if we only have the medium items or only have the large items, respectively.

So, for the original instance, the social cost of the optimal mechanism $\mathcal{A}^*$ is, by Smith's rule,
$$c(\mathcal{A}^*) = c^*(S)+n_MS+c^*(M)+n_R(S+M)+c^*(L)$$
Furthermore, for the original instance, the social cost of the $k$-th Pareto optimal scheduling mechanism is:
$$
c(\mathcal{A}^k) 
= c^*(S) + \lambda n_MS + \frac{1}{2}\lambda M(\lambda n_M+1)+ ((1-\lambda)n_M + n_L)(S+\lambda M)+\frac{1}{2}((1-\lambda)M+L)(n_L+(1-\lambda)n_M+1)
$$
Here the last term follows by (\ref{eq:fair}). 
Hence the welfare ratio is:
{\small
\begin{align}\label{eq:I-hat}
    \frac{c(\mathcal{A}^k)}{c(\mathcal{A}^*)} &= \frac{c^*(S) +  \lambda n_MS + \frac{1}{2}\lambda M( \lambda n_M+1) + ((1-\lambda )n_M + n_L)(S+\lambda M)+\frac{1}{2}((1-\lambda)M+L)(n_L+(1-\lambda)n_M+1)}{c^*(S)+n_MS+c^*(M)+n_L(S+M)+c^*(L)}\nonumber\\
&=\frac{\left(\frac{1}{2}\lambda M(\lambda n_M+1) + ((1-\lambda )n_M + n_L)\lambda M+\frac{1}{2}((1-\lambda)M+L)(n_L+(1-\lambda)n_M+1)\right)+ \left(c^*(S) + n_M S + n_L S\right)}{\left(c^*(M)+n_LM+c^*(L)\right) + \left(c^*(S) + n_M S + n_L S\right)}\nonumber\\
    &< \frac{\frac{1}{2}\lambda M(\lambda n_M+1) + ((1-\lambda )n_M + n_L)\lambda M+\frac{1}{2}((1-\lambda)M+L)(n_L+(1-\lambda)n_M+1)}{c^*(M)+n_LM+c^*(L)}
\end{align}
}%
The strict inequality follows from the fact that $\frac{\alpha+\gamma}{\beta+\gamma}< \frac{\alpha}{\beta}$, for any non-negative $\alpha, \beta, \gamma$ with $\alpha>\beta$.

But recall $S=0$ in $\hat{I}$. Thus the denominator of~(\ref{eq:I-hat}) is just the cost of the optimal schedule for $\hat{I}$. Similarly, the numerator is the social cost of the $k$-th Pareto schedule for $\hat{I}$.
Thus, as claimed, $\hat{I}$ has a worse welfare ratio than $I$ and 
every small job it contains has size $0$.
\end{proof}

{\sc Property 2.} There is exactly one large job in $\hat{I}$.

\begin{proof}
Recall we may assume every small job in $\hat{I}$ has size $0$ by Property 1.
But then these small jobs have no effect on the social cost (if they are scheduled first). Thus the welfare ratio~(\ref{eq:I-hat}) is unaffected by assuming there are no small jobs, that is $n_A=0$.

Then to maximize the efficacy ratio, we may minimize the denominator whilst keeping the numerator fixed. To do this, further construct $\hat{I}$ as follows.
Among the $n_L$ large jobs, take the $n_L-1$ smallest of them and change their sizes to $1$. That is, they become medium jobs.
Let the size of the remaining large job increase to $L-(n_L-1)$. 
Observe that these transformations imply the sizes of these jobs still sum to $L$. In particular, this implies that the numerator of~(\ref{eq:I-hat}) does remain the same, as desired.
The original optimal social cost is
\begin{equation}\label{eq:before}
c^*(M)+n_L M+c^*(L) 
\ = \ c^*(M)+n_L M + \left(\sum_{i=1}^{n_L-1}(n_L -i +1)\cdot d_i\right) + L
\end{equation}
Here the final two terms follow by (\ref{eq:opt}).
After the transformation, the optimal social cost is:
\begin{equation}\label{eq:after}
c^*(M)+n_L M + \sum_{i=1}^{n_L-1} (n_L -i +1)\cdot 1 + L
\end{equation}
Since $d_i > 1$ for each job, (\ref{eq:after}) is smaller than (\ref{eq:before}).
Thus, the optimal social cost has fallen, completing the proof.
\end{proof}

\subsection{The Efficacy Ratio}

So to calculate the worst case efficacy ratio it suffices to consider
an instance $\hat{I}$ where every job is of medium size, that is~$1$, except for one large job of size $d_n=L$, say.
Now let $n_0=M$ be the number of medium jobs. Thus $n=n_0+1$ and $D = n_0 + L$.
In addition, let $\hat \epsilon = \frac{n_0}{D}$ be the proportion of the total sum of the job sizes due to the medium jobs.
Finally, recall that $k$ of the medium jobs are in the first section for the $k$-th Pareto optimal schedule $\mathcal{A}^k$, and $n_0-k$ of the medium jobs are in the second section. Let $\epsilon=\epsilon_k=\frac{k}{D}$ signify the dividing point between the first section
and the second section; observe $0 \leq \epsilon \leq \hat \epsilon$.

\begin{theorem}\label{thm:upper-single}
The efficacy ratio of $\mathcal{A}^k$ is at most $\frac{1}{4\epsilon}+ \left(1+ \frac{\epsilon}{4}\right)$, where $\epsilon=\epsilon_k$.
\end{theorem}


\begin{proof}
\noindent The social cost of the $k$-th Pareto schedule is: 
\begin{align}
     c(\mathcal{A}^k) &= \frac{1}{2}\epsilon D(\epsilon D + 1) + (n-\epsilon D)\epsilon D + \frac{1}{2}(1-\epsilon)D(n-\epsilon D +1) \nonumber\\
    &= \frac{1}{2}\epsilon D\big(\epsilon D + 1 + 2(n-\epsilon D) - (n-\epsilon D +1)\big) + \frac{1}{2}D(n-\epsilon D + 1) \nonumber\\
    &= \frac{1}{2}\epsilon Dn + \frac{1}{2}D(n-\epsilon D + 1)  \nonumber\\
     &= \frac{1}{2}D\big( (\epsilon+1)n-\epsilon D + 1) \big) \nonumber
\end{align}
But $n=\hat \epsilon D +1$, so rearranging we have:
\begin{align}\label{eq:ratio-fair}
     c(\mathcal{A}^k) 
    &=\frac{1}{2}D\left(\epsilon+\hat \epsilon D - \epsilon(1-\hat \epsilon)D + 2 \right)
\end{align}
On the other hand, the social cost of the optimal solution is:
\begin{equation}\label{eq:ratio-opt}
  c(\mathcal{A}^*) \ =\  \sum_{j=1}^{N-1} j + D
  \ =\  \frac12(N-1)N + D
  \ =\  \frac{1}{2}D\cdot \left(\hat \epsilon^2 D + \hat \epsilon + 2\right)
\end{equation}
Then, from (\ref{eq:ratio-fair}) and (\ref{eq:ratio-opt}), the efficacy ratio is:
\begin{equation*}
    \frac{c(\mathcal{A}^k)}{c(\mathcal{A}^*)} 
    \ =\ \frac{\epsilon + \hat \epsilon D - \epsilon (1-\hat{\epsilon}) D + 2}{\hat \epsilon^2D + \hat \epsilon + 2}
    \ \leq\ \max \left( \frac{\epsilon + \hat \epsilon D - \epsilon (1-\hat{\epsilon})D}{\hat \epsilon^2D+\hat{\epsilon}}, \frac{2}{2} \right)
    \ = \ \frac{\epsilon + \hat \epsilon D - \epsilon (1-\hat{\epsilon})D}{\hat \epsilon^2D+\hat{\epsilon}}
\end{equation*}
Now recall that $\epsilon \le \hat \epsilon$. So assume $r\cdot \epsilon = \hat{\epsilon}$ for some $r\ge 1$. Thus

\begin{equation}
    \frac{c(\mathcal{A}^k)}{c(\mathcal{A}^*)} 
    \ \leq\ \frac{\epsilon + r\epsilon D - \epsilon (1-r{\epsilon}) D }{r^2\epsilon^2D + r \epsilon}
    \ =\ \frac{1 + r D - ({1}-r{\epsilon}) D}{r^2 \epsilon D + r}
\label{eq:ratio-fair-approx}
\end{equation}
To optimize this efficacy ratio, we must find the factor $r=r_{\max}$ that maximizes this upper bound. We first take the derivative of the ratio with respect to $r$:
\begin{align*}
    \frac{\partial }{\partial r} \left(\frac{1 + r D - ({1}-r{\epsilon}) D}{r^2 \epsilon D + r}\right) \ =\  \frac{D+\epsilon D}{r + \epsilon r^2D} - \frac{(1+2\epsilon r D)(1+rD-(1-\epsilon r)D)}{(r+\epsilon r^2 D)^2}
\end{align*}
Setting the derivative to $0$ and solving for $r$ with $D > 1$ and $0<\epsilon < \frac{D-1}{D}$ gives:
\begin{equation*}
r_{\max} 
    \ = \ \frac{D-1}{(1+\epsilon)D} + \sqrt{\frac{D - \epsilon D + \epsilon D^2 - 1}{\epsilon(1+\epsilon)^2D^2}}
     \ = \ \frac{1}{(1+\epsilon)D}\left( D-1 + \sqrt{\frac{D - \epsilon D + \epsilon D^2 - 1}{\epsilon}}\right)
\end{equation*}
Plugging $r_{\max}$ into the efficacy ratio (\ref{eq:ratio-fair-approx}), we find that the efficacy ratio is at most
\begin{equation*}
    \mathit{Eff}_{\max} := \frac{1 + r_{\max} D - ({1}-r_{\max}{\epsilon}) D}{r_{\max}^2 \epsilon D + r_{\max}}
\end{equation*}
We claim that $\mathit{Eff}_{\max} < \frac{(1+\e)^2}{4\e} + \frac{1}{2}=\frac{1}{4\epsilon}+ \left(1+ \frac{\epsilon}{4}\right)$. 
However, $\mathit{Eff}_{\max}$ is non-monotonic and awkward to analyze directly. So we will study it indirectly via $\mathit{Eff}_{\max}-\gamma$ where 
\begin{equation*}
    \gamma:= \frac{1}{r_{\max}^2 \epsilon D + r_{\max}}+\frac{(1+\e)}{4\e}
\end{equation*}
We will now show that
\begin{equation}
    \mathit{Eff}_{\max}-\frac{(1+\e)^2}{4\e}\ <\ \mathit{Eff}_{\max}-\gamma \ <\  \frac{1}{2}
    \label{eq:ratio-upper-tight}
\end{equation}
First, let's prove the lower bound in (\ref{eq:ratio-upper-tight}). It suffices to show $f<\frac{(1+\e)^2}{4\e}$, which is equivalent to showing:
\begin{equation}
    \frac{1}{r_{\max}^2 \epsilon D + r_{\max}} 
    \ <\  \frac{(1+\e)^2}{4\e} - \frac{(1+\e)}{4\e}
    \ < \ \frac{1+\e}{4}
     \label{eq:f-bound1}
\end{equation}
So to show (\ref{eq:f-bound1}) holds it suffices to show that $4 < ({1+\e})(r_{\max}^2 \epsilon D + r_{\max})$.
But 
\begin{equation*}
    ({1+\e})(r_{\max}^2 \epsilon D + r_{\max}) \ >\ (r_{\max}^2 \epsilon D + r_{\max}) \ \geq\  \e D 
\end{equation*}
Thus $\gamma<\frac{(1+\e)^2}{4\e}$, if $\frac{4}{\e} < D$. 
(In fact 
a more detailed argument shows that $\gamma<\frac{(1+\e)^2}{4\e}$ for $D>\frac{11+8\sqrt{2}}{7} \approx 3.19$ and $0<\e<1$.)

Next let's prove the upper bound in (\ref{eq:ratio-upper-tight}). 
For $D>1.41$ and $0<\e<1$, the function $\mathit{Eff}_{\max}-\gamma$ is monotonically increasing in both $\e$ and $D$. 
Therefore, an upper bound on $\mathit{Eff}_{\max}-\gamma$ can be obtained by letting $\e=1$ and $D\to \infty$.
In this case, $\hat{\e} = r = 1 $. So we have
\begin{equation*}
   \mathit{Eff}_{\max}-\gamma = \frac{rD-D+r\e D}{r^2\e D+r}-\frac{1+\e}{4\e} = \frac{D}{D+1}-\frac12 \to \frac12 
\end{equation*}
The theorem follows.
\qed
\end{proof}

\subsection{Lower Bound on the Efficacy Ratio} \label{lower_bound_ex}

So there is always a priority scheduling mechanism that is both 
$\epsilon$-fair and produces a solution whose social cost is within a factor
$\frac{1}{4\epsilon}+ \left(1+ \frac{\epsilon}{4}\right)$ of optimal.
This upper bound is tight to within an additive term of one half.
In particular, the dominant term $\frac{1}{4\epsilon}$ cannot be improved.
To see this we present an example for which any
priority scheduling mechanism that is $\epsilon$-fair must output a solution whose social cost is at least a factor
$\frac{1}{4\epsilon}+\frac{1}{2}$ greater than optimal.

\begin{restatable}{theorem}{lowersingle}\label{thm:lower-single}
The efficacy ratio of $\mathcal{A}^k$ is at least $\frac{1}{4\epsilon}+\frac{1}{2}$, where $\epsilon=\epsilon_k$.
\end{restatable}
\begin{proof}
To prove this lower bound, it suffices to present an example obtaining this bound. We attempt to construct the example in a way that forces the inequalities arising in the proof of the upper bound, namely Theorem~\ref{thm:upper-single}, to be tight. 
To do this, for any sufficiently large value of $D$ and arbitrary $\e$, we can choose explicitly $r=2$ such that $M = 2\e D$ and $L = D - M > 1$. This means that exactly half of the medium jobs will be in the first section and the remainder will be in the second section; see Figure~\ref{fig:lower bound example}. 

\begin{figure}[h]
    \centering
    \includegraphics[scale=1]{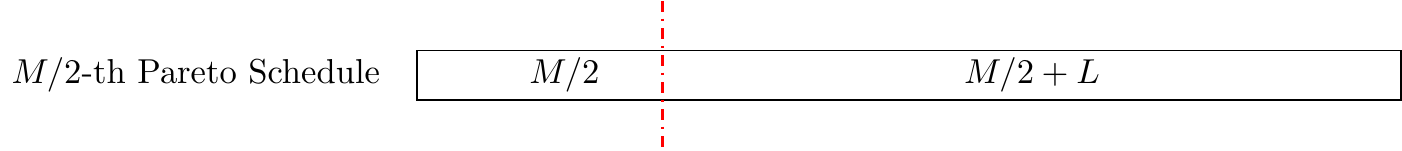}
\caption{The lower bound example.}
\label{fig:lower bound example}
\end{figure}
We then obtain:
$\frac{c(\mathcal{A}^k)}{c(\mathcal{A}^*)} = \frac{1 + rD - (1-r\epsilon)D}{r^2 \epsilon D + r} 
\ =\  \frac{1 + 2D- (1-2\epsilon) D}{2^2 \epsilon D + 2} 
\ =\ \frac{1 + D + 2\epsilon D}{4\epsilon D + 2}$.
When $D$ is sufficiently large, this tends to:
$\frac{c(\mathcal{A}^k)}{c(\mathcal{A}^*)} 
\ \rightarrow \  \frac{D + 2\epsilon D}{4\epsilon D}
\ =\ \frac{1 + 2\epsilon}{4\epsilon}
\ =\  \frac{(1+\epsilon)^2}{4\epsilon} - \frac{\epsilon}{4}
\ =\  \frac{1}{4\epsilon} + \frac{1}{2}.\qed$

\end{proof}

Theorem~\ref{thm:upper-single} and Theorem~\ref{thm:lower-single} apply for the case $\epsilon=\epsilon_k$.
But the bound of $\frac{1}{4\epsilon}$ is tight for any~$\epsilon>0$. In particular, we can choose any {\em target value} $\epsilon>0$ to be our fairness guarantee. 
Then we can still apply our worst case analysis.
This is because we defined $\epsilon_k=\frac{k}{D}$. Hence for any small $\epsilon$, there is a $k$ such that 
$\epsilon < \epsilon_k+\frac{1}{D}$. But recall $D\ge n$. Therefore, if we set a target fairness of $\epsilon$
then we can obtain an even better fairness guarantee of $\epsilon_k < \epsilon$ for no loss in the dominant term of $\frac{1}{4\epsilon}$ in the efficacy ratio (there is an insignificant loss only in the minor terms).

\section{Scheduling on Multiple Machines}\label{sec:multiple}

Ergo, for a single machine there is a scheduling algorithm that is both fair and efficacious. In this section we will extend this result to the 
case of multiple machines. Indeed, in this more complex setting, we will prove that exactly the same performance guarantees still hold.

\subsection{Optimal Scheduling on Multiple Machines}
Our first task is to compute an optimal scheduling mechanism $\mathcal{A}^*$ for minimizing the social cost with multiple machines. We will see that there is a simple algorithm to achieve this: {\em order the jobs by size and place the next job on the least crowded machine}. 
For the purposes of designing a fair algorithm, however, it is important to note that there is a class of randomized algorithms that also guarantee an optimal scheduling. Begin by adding dummy jobs of size zero to ensure there are exactly $n=m\cdot \tau$ jobs. Next, order the jobs by size and then group them into blocks of size $m$. Finally, for each block, independently generate a random perfect matching to assign the $m$ jobs in the block to the $m$ machines.

To analyze this, let $M_r$ be the total size of the jobs in the $r$th block. The social cost is then
$\tau\cdot M_1 +(\tau-1)\cdot M_2 + \cdots +2\cdot M_{\tau-1} + M_{\tau}$.
This is because the jobs in block $r$ are delayed in total by the jobs before them. That this randomized algorithm is optimal follows from the fact 
its deterministic version is optimal, a fact observed by Baker~\cite{Bak74}.
We include a proof for completeness.
\begin{theorem}\label{thm:opt-multiple}
The optimal schedule has cost $\tau\cdot M_1 +(\tau-1)\cdot M_2 + \cdots +2\cdot M_{\tau-1} + M_{\tau}$.
\end{theorem}
\begin{proof}
Take the optimal solution $\mathcal{A}^*$. We must ensure that $\mathcal{A}^*$ assigns each job in the last block (call these large jobs) to a different machine.
Suppose not then, without loss of generality, Machine 1 has two large jobs (of size $d_1$ and $d_2$) and Machine 2 has no large jobs.
Let the largest job on Machine 2 have size $d_0$; this job exists or the schedule is evidently non-optimal. So $d_0< d_1\le d_2$; observe that the first inequality is strict otherwise we could have used job $0$ as a large job instead of job $1$. Now swap the assignment of job $0$ and job $1$ to create a new schedule $\hat{\mathcal{A}}$. The change in social cost is exactly $d_0-d_1$ because job $2$ is now delayed by a shorter job. Indeed, let $t_1$ denote the finishing time of Machine 1 without considering jobs 1 and 2, and let $t_2$ denote the finishing time of Machine 2 without considering job 0. The sum of social costs of these three jobs before the swap is given by:
$(t_1 + d_1) + (t_1 + d_1 + d_2) + (t_2 + d_0)$.
After the swap this becomes:
$(t_1 + d_0) + (t_1 + d_0 + d_2) + (t_2 + d_1)$.
This argument is illustrated in Figure~\ref{fig:swap}. We conclude that $c(\hat{\mathcal{A}}) < c(\mathcal{A}^*)$, since $d_0 - d_1 < 0$, contradicting the optimality of $\mathcal{A}^*$.
\begin{figure}[h!]
\centering
\includegraphics[scale=0.8]{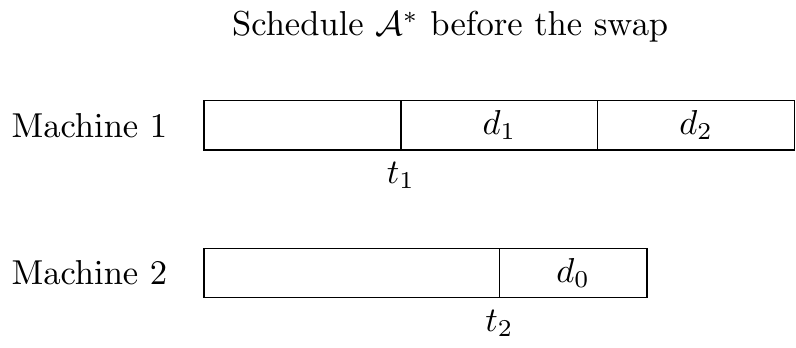}
\hspace{0.5cm}
\includegraphics[scale=0.8]{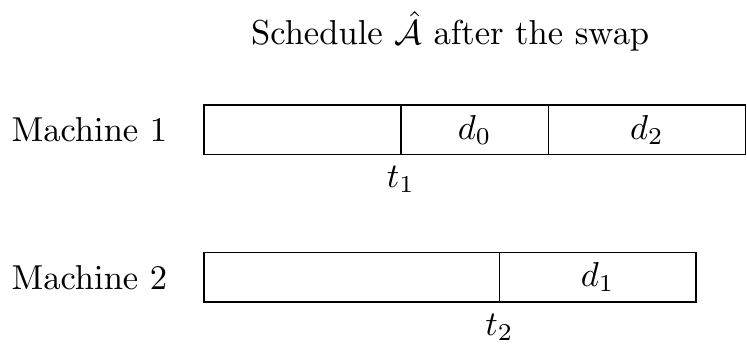}
\caption{Swapping the jobs $d_0$ and $d_1$ yields a Pareto improvement}\label{fig:swap}
\end{figure}

We can now apply induction to show each other block also has its jobs assigned to different machines. 
Thus the cost of the optimal solution is also $\tau\cdot M_1 +(\tau-1)\cdot M_2 + \cdots +2\cdot M_{\tau-1} + M_{\tau}$.
\end{proof}

As alluded to, Theorem~\ref{thm:opt-multiple} implies that there is 
a simple greedy algorithm $\mathcal{A}^*$ to output the optimal schedule.
Simply order the jobs by size and, at each step, place the next job on 
the least crowded machine. 
Observe this greedy algorithm has the property of assigning the jobs in a block to different machines. This is because when a job is assigned to a machine with the lowest load, that machine now has the highest load because the jobs are ordered by size. This can be shown inductively. Indeed, let Machine~1 be the  machine with the lowest load denoted $\ell_1$. Let Machine~2 be the machine with the highest load to which we have just added a job of size $d_2$ and which previously had the lowest load $\ell_2$. So in particular, we have that $\ell_2\le \ell_1$. After adding the next job of size $d_1 \ge d_2$ to Machine~1, it must have higher load than Machine~2, since $\ell_1 + d_1 \ge \ell_2 + d_2$. Consequently, Machine~1 now has the highest load. It follows that every other machine must then receive a job before that machine receives another job.
Hence the greedy algorithm can be viewed as scheduling the blocks in order
via perfect matchings and thus outputs an optimal schedule.
\subsection{Fair Scheduling on Multiple Machines}
Of course, as with a single machine, the optimal schedules for multiple machines may be very unfair. But, we can use the ideas developed for the single machine case to design a fair and efficacious algorithm for the multiple machine case.
First we must define the fairest scheduling mechanism $\mathcal{A}^f$.
Recall a mechanism is fair if it does not use type data. But now, for the multiple machine setting,
there is no longer a unique fair mechanism. For example, scheduling all the jobs (in random order) 
on just one of the $m$ machines is fair! Thus, we define the fairest mechanism $\mathcal{A}^f$ to be the mechanism that minimizes social cost over the set of all fair mechanisms.
In particular, $\mathcal{A}^f$ will assign each job to a random machine and then, for each machine, 
order the jobs assigned to it at random.

Next we must define our scheduling mechanism $\mathcal{A}$. 
Specifically, we use the randomized perfect matching algorithm to assign
$\tau$ jobs to each machine but not to order the jobs on those machines. Instead
to order the jobs on each individual machine, we use the $\epsilon$-fair scheduling algorithm 
developed in Section~\ref{sec:ratio} for a single machine.
We claim $\mathcal{A}$ is $\epsilon$-fair and has efficacy ratio $\Theta(\frac{1}{\e})$
for the multiple machine-scheduling problem.

Let's first show that the $\epsilon$-fairness guarantee still applies.
\begin{theorem}\label{thm:multi-fair}
The multiple-machine scheduling mechanism $\mathcal{A}$ is $\epsilon$-fair.
\end{theorem}
\begin{proof}
Take any job $i$ belonging to block $r$ which, recall, has total size $M_r$. Let $j(i)$ be the machine that job $i$ is assigned to by the mechanism. 
By construction, no jobs in the block $r$ can delay job $i$, because they must be assigned to a different machine. Furthermore, by the randomization, any job in a lower or higher block is on machine $j(i)$ with probability exactly $1/m$. Hence,
the expected load on machine $j(i)$ is $\frac{D - M_r}{m}+d_i$. 

By Lemma~\ref{lem:fairness-last}, for the worst case fairness we may assume job $i=k+1$ is the $i$th smallest job on machine $j$ and we use the $k$-th Pareto schedule $\mathcal{A}^k$ where
$\epsilon = \epsilon_k = \frac{\sum_{\ell=1}^k d_\ell}{\sum_{\ell=1}^{\tau} d_\ell}$.
Thus the expected cost for job $i$ is
$$c_i(\mathcal{A}) \ = \ \epsilon \left(\frac{D-M_r}{m}\right) + \frac{1}{2} \left( (1- \epsilon) \frac{D-M_r}{m}\right) + d_i.$$
On the other hand, under the fairest mechanism $\mathcal{A}^f$ each task will be assigned to a random machine in a random order. The expected cost for job $i$ with this optimal fairness scheme is 
$c_i(\mathcal{A}^f) \ = \ \frac{D-d_i}{2m}$.
Thus the fairness of mechanism $\mathcal{A}$ to agent $i$ is
\begin{equation*}
    f_i(\mathcal{A})
    = \frac{c_i(\mathcal{A})}{c_i(\mathcal{A}^f)} 
    \ =\ \frac{\epsilon \left(\frac{D-M_r}{m}\right) + \frac{1}{2} \left( (1- \epsilon) \frac{D-M_r}{m}\right) + d_i}{\frac{D-d_i}{2m}+d_i} 
    \le \frac{(1+ \epsilon) \frac{D-M_r}{m} + 2d_i}{\frac{D-M_r}{m}+2d_i} 
    \leq 1+\epsilon.
\end{equation*}
Here the first inequality holds because $d_i\le M_r$, as job $i$ belongs to block $r$.
Thus, the mechanism is $\epsilon$-fair. 
\qed
\end{proof}

\subsection{The Efficacy Ratio for Multiple Machines}
So the proposed scheduling mechanism is $\mathcal{A}$ is $\epsilon$-fair.
Let's now prove that dominant term in the efficacy ratio for $\mathcal{A}$ is 
again exactly $\frac{1}{4\epsilon}$.
\begin{theorem}\label{thm:multi-efficacy}
For machine scheduling on multiple machines, the mechanism $\mathcal{A}$ has efficacy ratio at 
most~$\frac{1}{4\epsilon}+ \left(1+ \frac{\epsilon}{4}\right)$.
\end{theorem}
\begin{proof}
By Theorem~\ref{thm:upper-single}, applying the Pareto optimal scheduling increases the
social cost on each machine by a factor at most $\frac{1}{4\epsilon}+ \left(1+ \frac{\epsilon}{4}\right)$.
But the initial assignment of jobs to machines was optimal by Theorem~\ref{thm:opt-multiple}. Thus mechanism outputs a schedule whose 
total social cost over all machines is at most a factor $\frac{1}{4\epsilon}+ \left(1+ \frac{\epsilon}{4}\right)$ greater than optimal.
\qed
\end{proof}
Again the dominant term $\frac{1}{4\epsilon}$ applies even when the targeted fairness guarantee $\epsilon$
does not coincide with any of the $\epsilon_k$. 
(Note, given a target fairness of $\epsilon$, the mechanism $\mathcal{A}^k$ used may vary on different machines.)
Moreover, this efficacy ratio is tight because
the single-machine setting is a special case of the multiple-machine setting;
so the lower bound of $\frac{1}{4\epsilon}+\frac12$ from Theorem~\ref{thm:lower-single} applies to the multiple-machine setting.

\section{Conclusion}
For the machine scheduling problem, we have shown that the dichotomy between fairness and efficacy can be overcome by allowing for a negligible amount of bias in the mechanism. We conjecture that this paradigm extends to a much broader range of applications and also to other classical models in computer science such as matching, the assignment problem, flows and network routing. Accordingly, investigating this hypothesis
is the most important line of research arising from this paper.

\bibliographystyle{splncs04}
\bibliography{sagtbib}

\end{document}